\documentclass[12pt]{article}
\usepackage{amsmath,amssymb,amsthm,enumitem}
\usepackage[margin=1.2in]{geometry}
\usepackage[numbers]{natbib}
\usepackage{tikz}
\usepackage{graphicx}
\newtheorem{theorem}{Theorem}
\newtheorem{lemma}{Lemma}
\newtheorem{definition}{Definition}

\title{A Relative Ignorability Framework for Decision-Relevant Observability in Control Theory and Reinforcement Learning}
\author{MaryLena Bleile$^{**}$, Minh-Nhat Phung$^\dagger$ and Minh-Binh Tran$^\dagger$\footnote{M.-N. P and M.-B. T are funded in part by the NSF Grants DMS-2204795, DMS-2305523, Humboldt
Fellowship, NSF CAREER DMS-2303146, DMS-2306379}\\\\\\
$^{**}$Sanofi Pharmaceuticals, New York, NY 10018, USA\\
Email: MaryLena.Bleile@sanofi.com \\
Permanent Email: marylenableile@gmail.com
\\\\
$^\dagger$Department of Mathematics, Texas A\&M University,\\ College Station, TX 77843, USA\\
Emails: pmnt1114@tamu.edu \& minhbinh@tamu.edu\\
}

\date{\today}

\begin{document}

\maketitle

\begin{abstract}
Sequential decision-making systems routinely operate with missing or incomplete data. Classical reinforcement learning theory, which is commonly used to solve sequential decision problems, assumes Markovian observability, which may not hold under partial observability. Causal inference paradigms formalise ignorability of missingness. We show these views can be unified and generalized in order to guarantee Q-learning convergence even when the Markov property fails. To do so, we introduce the concept of \emph{relative ignorability}. Relative ignorability is a graphical-causal criterion which refines the requirements for accurate decision-making based on incomplete data. Theoretical results and simulations both reveal that non-markovian stochastic processes whose missingness is relatively ignorable with respect to causal estimands can still be optimized using standard Reinforcement Learning algorithms. These results expand the theoretical foundations of safe, data-efficient AI to real-world environments where complete information is unattainable.
\end{abstract}

\section{Introduction}

Reinforcement learning theory traditionally assumes that agents have complete access to state information \citep{sutton2018, bellman1957}. However, real-world applications usually involve missing or unobserved state components, even if these are not explicitly acknowledged or modelled. Consider, for example, the clinical AI agent developed by Komorowski, et al \cite{komorowski2018artificial}, which was shown to provide treatment suggestions for sepsis care. The AI clinician makes decisions based on 48 clinical features, aiming to minimize overall and 90-day hospital mortality rates. The input feature set, which included demographic data, vital signs, lab values, and medication history, notably does not include medical insurance status, which has been shown to influence sepsis outcome even after controlling for the hospital treatments received \cite{kumar2014association,rhee2019prevalence}. In the absence of complete information, therefore, the standard convergence theorems do not apply. 

\begin{figure}[hb]
\centering
\includegraphics[width=350pt, height=180pt]{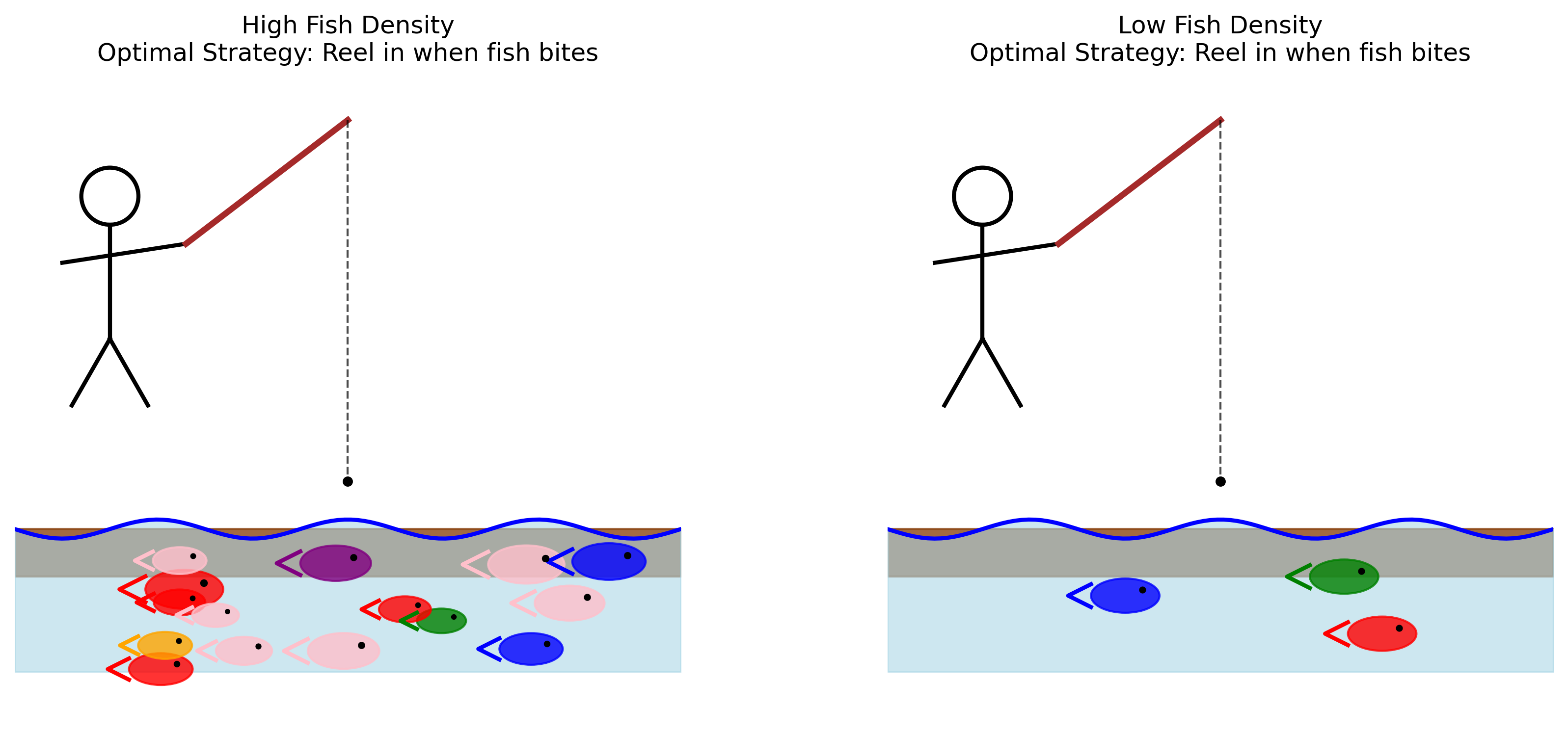}
\label{fig:fish}
\caption{Fish abundance and type in the pond may affect the probability of reward. However, the optimal fishing strategy is the same regardless of fish abundance. Hence, fish abundance is relatively ignorable with respect to fishing strategy.}
\end{figure}

While Partially Observable Markov Decision Processes (POMDPs) address the challenge of partial observability \citep{puterman1994}, they require computationally expensive belief state maintenance, and convergence proofs typically require that the latency structure of the POMDP model is accurate.  However, explicit modelling and estimation of all latent variables may not be necessary: Missing state components may not affect optimal decision-making, even when they violate the Markov property. Consider, for example, the fishing scenario described by Krause and H{\"u}botter \cite{krause2025probabilistic}, where we wish to reel in the line only if there is a fish on the hook, and not otherwise. In this situation, we are rewarded for each successful catch of fish, with no reward acquired for reeling in an empty line. This problem does not explicitly model the composition of fish types in the pond, which may affect the overall rate of fish catching, since some fish may bite more frequently than others. However, the optimal strategy for line-reeling is to reel in the line every time there is a fish on it, regardless of type; hence, the fish composition variable is not important for decision-making, and does not need to be included in the POMDP model. Figure \ref{fig:fish} demonstrates this idea; regardless of how many fish are in the pond, one should still pull the reel in if a fish bites, and leave the hook in the water otherwise.

A similar phenomenon is present in the sepsis treatment model; here, an individual's insurance status does not affect the ranking of the potential treatments one could apply. Suppose, for example, that drug A is better than drug B for an insured individual. One might reasonably assume that if that same individual were uninsured, then drug A would still work better than drug B, even though the individual's probability of mortality would increase regardless of whether the individual received drug A or drug B; we say that insurance status is \emph{relatively ignorable} with respect to the treatment action. With this in mind, it is unsurprising that the AI clinician which ignored patients' insurance status was highly effective despite the fact that standard Reinforcement Learning convergence proofs do not apply: Indeed, the AI clinician outperformed human clinicians by a substantial margin. 

The current paper provides the mathematical scaffolding which justifies the use of Q-learning in \citep{komorowski2018artificial}: Our main contribution is a  novel concept of \emph{relative ignorability}; this definition is foundational on established concepts developed in statistical literature on causal inference and missing data. Based on the novel definition, we provide a novel proof that Q-learning \citep{watkins1989} converges under a marginal Bellman operator when missing components satisfy this condition. Finally, we discuss the potential utility of the relative ignorability concept beyond classic Q-learning, showing potential extensions to Deep Q-learning, as well as applications in clinical dosing and fault tolerance in distributed systems.

\section{Background}
Our novel concept of relative ignorability draws from missing data theory in causal inference \citep{diggle1994, mohan2013}. In causal inference, marginal structural models \citep{robins2000} handle time-varying confounding through reweighting schemas that share mathematical structure with POMDP estimation.

In classic statistical literature, there are three types of missing data: Missingness can be i) \emph{missing completely at random}, where missingness is not related to outcome, ii) \emph{missing at random}, where missingness and outcome are unrelated after controlling for the observed variables, or iii) \emph{non-ignorable}, where missingness is associated with outcome even after controlling for observed covariates. 

A seminal paper by Diggle and Kenward \cite{diggle1994} identified non-ignorable missingness in three real-world datasets. One of these datasets was a trial which aimed to compare different diets for dairy cows: The outcome vector consisted of milk protein concentrations at a series of timepoints in the study. Some cows had to drop out of the study because they stopped producing milk, inducing a degree of missingness to the data. This was shown to be \emph{informative dropout}, which means that the resulting missingness in the data was non-ignorable. However, the informative dropout model yielded the same actionable insight as the model which ignored the informative dropout. 

Actionable insights can be stable despite non-ignorable missingness if the missing variables affect treatment groups equally. Our novel concept of \emph{relative ignorability} allows one to differentiate the stable action-insight case from other situations where inclusion of the missing data would lead us to a different actionable insight. We propose that it is only strictly necessary to model the missing information under the second case where missing variables are relatively ignorable, proving that sequential decision-making (Q-learning) based on incomplete information with relatively ignorable missingness can still yield an optimal policy. This theoretical result explains the efficacy of many practical applications such as the AI clinician discussed previously. 

\section{Preliminaries}

Let $(\Omega,\mathcal{F},\mathbb{P})$ be a complete probability space.  Over the probability space, we consider a stationary Markov Decision Process $(\mathcal{X}, \mathcal{A}, \Pi, \mu, \Gamma, \rho, \gamma)$ where:
\begin{itemize}
	\item The countable set $\mathcal{X}\subset\mathbb{R}^{d}$ is the \emph{state space}, and we denote the set of subsets of $\mathcal{X}$ by $\mathcal{B}(\mathcal{X})$.
	\item The \emph{whole action space} is a countable set $\mathcal{A}\subset\mathbb{R}^{d'}$.
We denote by $\mathcal{B}(\mathcal{A})$ the set of subsets of $\mathcal{A}$ and $\mathcal{P}(\mathcal{A})$ the set of probability measures on $\mathcal{A}$.

	\item We consider the collection of random variables $X_j:\Omega\to\mathcal{X}$, which represent the state, and $A_j:\Omega\to\mathcal{A}$, which represent the action.
	\item For a set $\mathcal{Z}$ ($\mathcal{X}$ or $\mathcal{A}$) and random variables $Z_k:\Omega\to \mathcal{Z}$ ($X_j$ or $A_j$), we recall the notation of conditional probability
		\begin{align*}
			\mathbb{P}(Z_k\in\mathcal{Z}_k\mid Z_{k-1}=z_{k-1},\dots,Z_{0}=z_0)=\frac{\mathbb{P}(Z_k^{-1}(\mathcal{Z}_k)\cap Z_{k-1}^{-1}(z_{k-1})\cap\dots\cap Z_0^{-1}(z_0))}{\mathbb{P}(Z_{k-1}^{-1}(z_{k-1})\cap\dots\cap Z_0^{-1}(z_0))},
		\end{align*}
		if
		\[\mathbb{P}(Z_{k-1}^{-1}(z_{k-1})\cap\dots\cap Z_0^{-1}(z_0))>0.\]
	\item The class of \emph{decision policies} $\Pi$ is a set of functions $\pi:\mathcal{X}\to\mathcal{P}(\mathcal{A})$. The policy satisfies that the support of $\pi(\cdot\mid x)$ is a set $\mathcal{A}_{x}\subset\mathcal{A}$.

		The set $\mathcal{A}_{x}$ is the space of \emph{allowable actions} for state $x$, meaning
\[\mathbb{P}(A_j\in\widetilde{\mathcal{A}}\mid X_j=x )>0,\quad\forall j\ge0,x\in\mathcal{X},\widetilde{\mathcal{A}}\in\mathcal{B}(\mathcal{A}), \widetilde{\mathcal{A}}\cap\mathcal{A}_{x}\ne\emptyset.\]
Since $\mathcal{X},\mathcal{A}$ are countable, we have the following countable set
\[\mathcal{Y}=\bigcup_{x\in\mathcal{X}}\{x\}\times \mathcal{A}_{x}\subset\mathcal{X}\times\mathcal{A}.\]

The policy $\pi$ represents the probability to take an action in $\mathcal{A}$ given state $x$, in the sense that
		\[\mathbb{P}(A_j\in\widetilde{\mathcal{A}}\mid X_j=x_j)=\int_{\widetilde{\mathcal{A}}}\pi(da_j\mid x_j)=\pi(\widetilde{\mathcal{A}}\mid x_j)\]
		for all $\widetilde{\mathcal{A}}\in\mathcal{B}(\mathcal{A}),x_j\in\mathcal{X}$.

	\item The distribution $\mu:\mathcal{B}\to[0,1]$ is the \emph{initial state distribution}, that is for $\widetilde{\mathcal{X}}\in\mathcal{B}(\mathcal{X})$
		\[\mu(\widetilde{\mathcal{X}})=\int_{\widetilde{\mathcal{X}}}\mu(dx_0)=\mathbb{P}(X_{0}\in\widetilde{\mathcal{X}}).\]
		We will consider $\mu(\{x\})>0$ for all $x\in\mathcal{X}$.
	\item The function $\Gamma: \mathcal{B}(\mathcal{X})\times\mathcal{Y} \to [0,1]$ is the \emph{transition kernel}, that satisfies
		\[\Gamma(\widetilde{\mathcal{X}}\mid x_j,a_j)=\int_{\widetilde{\mathcal{X}}}\Gamma(dx_{j+1}\mid x_j,a_j)=\mathbb{P}(X_{j+1}\in\widetilde{\mathcal{X}}\mid X_{j}= x_j,A_j=a_j)\]
		for all $\widetilde{\mathcal{X}}\in\mathcal{B}(\mathcal{X}), (x_j,a_j)\in\mathcal{Y}$.
		Similar to the initial state distribution,  we consider for all $j\ge0,x_{j+1}\in\mathcal{X}$, there exists $(x_j,a_j)\in\mathcal{Y}$ such that $\Gamma(\{x_{j+1}\}\mid x_j,a_j)>0$.
	\item The Markov property for the decision process includes:
		\[\mathbb{P}(A_{j}\in\widetilde{\mathcal{A}}\mid X_{j}=x_j,A_{j-1}=a_{j-1},\dots,X_{0}=x_0)=\mathbb{P}(A_j\in\widetilde{\mathcal{A}}\mid X_j=x_j),\]
		and
		\[\mathbb{P}(X_{j+1}\in\widetilde{\mathcal{X}}\mid X_{j}=x_j,A_j=a_j,\dots,X_0=x_0,A_0=a_0)=\mathbb{P}(X_{j+1}\in\widetilde{\mathcal{X}}\mid X_{j}=x_j,A_j=a_j)\]
		for all $j\ge0$, where $\widetilde{\mathcal{A}}\in\mathcal{B}(\mathcal{A})$, $\widetilde{\mathcal{X}}\in\mathcal{B}(\mathcal{X})$, and $(x_k,a_k)\in\mathcal{Y}$ for $k=0,\dots,j$.
	\item The bounded function $\rho: \mathcal{Y}\to \mathbb{R}$ is the \emph{reward function}.
	\item The constant $\gamma\in(0,1)$ is the \emph{discount factor}.
\end{itemize}

{\bf Example:} We consider the simple system of 2 states $\mathcal{X}=\{0,1\}$ and the action space $\mathcal{A}=\{0,1,2\}$.
  
We consider the allowable action spaces $\mathcal{A}_0=\{0,1\}$ and $\mathcal{A}_1=\{2\}$. Thus, $\mathcal{Y}=\{(0,0),(0,1),(1,2)\}$. The initial is given by $\mu(\{0\})=0.5$, and $\mu(\{1\})=0.5$. We want to transition from state $0$ to state $1$. Let the transition kernel be given by
\[\Gamma(\{0\}\mid 0,0)=0.6,\quad\Gamma(\{1\}\mid 0,0)=0.4,\quad\Gamma(\{0\}\mid 0,1)=0.1,\]\[\quad\Gamma(\{1\}\mid 0,1)=0.9,\quad\Gamma(\{1\}\mid (1,2))=1.\]
We see that once we reach state $1$, we stop.
We consider a reward as follows:
\[\rho(0,0)=1,\quad\rho(0,1)=2,\quad\rho(1,2)=0.\]
This means we give reward for any allowable action when in state $0$ but give none when we already at state $1$.

We can try to enforce the policy $\pi$, which is given by
\[\pi(\{0\}\mid 0)=0.2,\quad\pi(\{1\}\mid 0)=0.8,\quad\pi(\{2\}\mid 1)=1.\]
Or, we consider $\pi$, which is given by
\[\pi(\{0\}\mid 0)=0.8,\quad\pi(\{1\}\mid 0)=0.2,\quad\pi(\{2\}\mid 1)=1.\]
The class of policies $\Pi$ consists of these two policies, and we want to know which is better given $\rho$ and $\gamma$.

By \cite[Proposition 7.28]{bertsekas1978stochastic}, there exists unique probability measure $\mu_j^\pi$ on $(\mathcal{X}\times\mathcal{A})^{j}\times\mathcal{X}$, which is supported in $\mathcal{Y}^{j}\times\mathcal{X}$, such that for all $\widetilde{\mathcal{X}}_j\in\mathcal{B}(\mathcal{X})$ and $\widetilde{\mathcal{A}}_j\in\mathcal{B}(\mathcal{A})$, we have
\begin{align*}
	\mu_j^\pi(\widetilde{\mathcal{X}}_0\times\widetilde{\mathcal{A}}_0\times\dots\times\widetilde{\mathcal{A}}_{j-1}\times\widetilde{\mathcal{X}}_j)=\int_{\widetilde{\mathcal{X}}_0}\int_{\widetilde{\mathcal{A}}_0}\dots\int_{\widetilde{\mathcal{X}}_j}\Gamma(dx_j\mid x_{j-1},a_{j-1})\pi(da_{j-1}\mid x_{j-1})\dots\pi(da_0\mid x_0)\mu(dx_0).
\end{align*}
There also exists unique probability measure $\bar{\mu}_j^\pi$ on $(\mathcal{X}\times\mathcal{A})^{j+1}$, which is supported in $\mathcal{Y}^{j+1}$, such that
\begin{align*}
	\bar{\mu}_j^\pi(\widetilde{\mathcal{X}}_0\times\widetilde{\mathcal{A}}_0\times\dots\times\widetilde{\mathcal{X}}_j\times\widetilde{\mathcal{A}}_{j})=\int_{\widetilde{\mathcal{X}}_0}\int_{\widetilde{\mathcal{A}}_0}\dots\int_{\widetilde{\mathcal{A}}_j}\pi(da_{j}\mid x_{j})\Gamma(dx_{j}\mid x_{j-1},a_{j-1})\dots\pi(da_0\mid x_0)\mu(dx_0).
\end{align*}
According to \cite[Proposition 7.28]{bertsekas1978stochastic}, by Kolmogorov extension theorem, there exists a unique extension $\bar{\mu}_\infty^\pi$ of $\mu_j^\pi$ and $\bar{\mu}_j^\pi$ on $(\mathcal{X}\times\mathcal{A})^{\mathbb{N}}$, supported in $\mathcal{Y}^{\mathbb{N}}$. It is an extension in the sense that
	\begin{align*}
		\mu_j^\pi(\widetilde{\mathcal{X}}_0\times\widetilde{\mathcal{A}}_0\times\dots\times\widetilde{\mathcal{A}}_{j-1}\times\widetilde{\mathcal{X}}_j)=\int_{\widetilde{\mathcal{X}}_0\times\widetilde{\mathcal{A}}_0\times\dots\times\widetilde{\mathcal{X}}_j\times\mathcal{A}\times(\mathcal{X}\times\mathcal{A})^{\mathbb{N}\setminus\{0,1,\dots,j\}}}\bar{\mu}_\infty^\pi,\quad\forall j\ge0;
	\end{align*}
	and
\begin{align*}
	\bar{\mu}_{j+1}^\pi(\widetilde{\mathcal{X}}_0\times\widetilde{\mathcal{A}}_0\times\dots\times\widetilde{\mathcal{X}}_j\times\widetilde{\mathcal{A}}_{j})=\int_{\widetilde{\mathcal{X}}_0\times\widetilde{\mathcal{A}}_0\times\dots\times\widetilde{\mathcal{X}}_j\times\widetilde{\mathcal{A}}_{j}\times(\mathcal{X}\times\mathcal{A})^{\mathbb{N}\setminus\{0,1,\dots,j\}}}\bar{\mu}_\infty^\pi,\quad\forall j\ge0.
\end{align*}

For $\bar{\mu}^\pi_\infty$, we have the following lemma.

\begin{lemma}
	For $\widetilde{\mathcal{X}}_j\subset\mathcal{X},\widetilde{\mathcal{A}}_j\subset\mathcal{A}$ and $\widetilde{\mathcal{X}}_j\times\widetilde{\mathcal{A}}_j\cap\mathcal{Y}\ne\emptyset$, we have
	\[\bar{\mu}_{\infty}^\pi((\mathcal{X}\times\mathcal{A})^{j}\times\widetilde{\mathcal{X}}_j\times\widetilde{\mathcal{A}}_j\times(\mathcal{X}\times\mathcal{A})^{\mathbb{N}\setminus\{0,1,\dots,j\}})>0.\]
\end{lemma}
\begin{proof}
	It suffices to show the proof for the case $\widetilde{\mathcal{X}}_j=\{x\},\widetilde{\mathcal{A}}_j=\{a\}$ for $(x,a)\in\mathcal{Y}$. For the mentioned case, we have
	\begin{align*}
		&\bar{\mu}_{\infty}^\pi((\mathcal{X}\times\mathcal{A})^{j}\times\widetilde{\mathcal{X}}_j\times\widetilde{\mathcal{A}}_j\times(\mathcal{X}\times\mathcal{A})^{\mathbb{N}\setminus\{0,1,\dots,j\}})\\
		&=\bar{\mu}_{j+1}^\pi((\mathcal{X}\times\mathcal{A})^{j}\times\{x\}\times\{a\})\\
		&= \mathbb{P}(X_{j}=x,A_{j}=a)\\
		&= \mathbb{P}(A_{j}=a\mid X_{j}=x)\mathbb{P}(X_j=x)\\
		&= \mathbb{P}(A_{j}=a\mid X_{j}=x)\sum_{(x',a')\in\mathcal{Y}}\mathbb{P}(X_j=x\mid X_{j-1}=x',A_{j-1}=a')\mathbb{P}(X_{j-1}=x',A_{j-1}=a' ).
	\end{align*}
	From the equality, we can obtain the positivity by induction.
\end{proof}
For $\widetilde{\mathcal{X}}_j\subset\mathcal{X},\widetilde{\mathcal{A}}_j\subset\mathcal{A}$ and $\widetilde{\mathcal{X}}_j\times\widetilde{\mathcal{A}}_j\cap\mathcal{Y}\ne\emptyset$, we can define a probability measure $\bar{\mu}_{j+1,\infty}^\pi(\cdot\mid \widetilde{\mathcal{X}}_{j},\widetilde{\mathcal{A}}_{j})$ over $(\mathcal{X}\times\mathcal{A})^{\mathbb{N}\setminus\{0,1,\dots,j\}}$ as follows
	\begin{align*}
		\bar{\mu}_{j+1,\infty}^\pi(\widetilde{\mathcal{Y}}_{j+1,\infty}\mid \widetilde{\mathcal{X}}_j,\widetilde{\mathcal{A}}_j)=\frac{\bar{\mu}_{\infty}^\pi( (\mathcal{X}\times\mathcal{A})^{j}\times\widetilde{\mathcal{X}}_j\times\widetilde{\mathcal{A}}_j\times\widetilde{\mathcal{Y}}_{j+1,\infty})}{\bar{\mu}_{\infty}^\pi((\mathcal{X}\times\mathcal{A})^{j}\times\widetilde{\mathcal{X}}_j\times\widetilde{\mathcal{A}}_j\times(\mathcal{X}\times\mathcal{A})^{\mathbb{N}\setminus\{0,1,\dots,j\}})},
	\end{align*}
	for all Borel measurable $\widetilde{\mathcal{Y}}_{j+1,\infty}$ of $(\mathcal{X}\times\mathcal{A})^{\mathbb{N}\setminus\{0,1,\dots,j\}}$. 

	For bounded measurable function $g:\mathcal{Y}^{\mathbb{N}\setminus\{0,\dots,j\}}\to\mathbb{R}$, we define
	\begin{align*}
		\mathbb{E}_{j+1,\infty}^\pi[g\mid \widetilde{\mathcal{X}}_j,\widetilde{\mathcal{A}}_j]=\int_{(\mathcal{X}\times\mathcal{A})^{\mathbb{N}\setminus\{0,1,\dots,j\}}}g\bar{\mu}_{j+1,\infty}^\pi(\cdot\mid\widetilde{\mathcal{X}}_j,\widetilde{\mathcal{A}}_j).
	\end{align*}

	We now discuss the missing data model. In the Markov Decision Process, for each $j\ge0$, we get the state $x_j\in\mathcal{X}$, the action $a_j\in\mathcal{A}$ is then chosen to determine the reward $\rho$ and transition $\Gamma$. In the missing data model, we only have partial information of the state. At time $j$, we consider a countable set $\mathcal{X}_j^o\subset\mathbb{R}^{d_j}$, where $d_j\le d$ as the observed state space. The observation is characterized by the index set $I_j^o\subset\{1,\dots,d\}$ and a projection map $\mathcal{O}_j:\mathcal{X}\to\mathcal{X}_{j}^o,(x^i)_{i=1}^d\mapsto (x^i)_{i\in I_j^o}$.

\medskip

Now, we introduce the new concepts of \emph{partially ignorability} and \emph{relative ignorability}.

\begin{definition}[Partial Ignorability]
	The Markov Decision Process is called partially ignorable at time $j+1$ if the following conditions are met:
	\begin{itemize}
		\item There exists a partition $\{1,\dots,d\}$ into $I_U,I_W$ so that the random variable $X_{j+1}=(X_{j+1}^i)_{i=1}^{d}$, where $U=(X_{j+1}^i)_{i\in I_U}, W=(X_{j+1}^i)_{i\in I_W}$ are two independent random variables.
			The independence in here means that there exist distributions $\Gamma_U,\Gamma_W$ such that
			\[\Gamma_U(\mathcal{U}\mid x,a)\Gamma_W(\mathcal{W}\mid x,a)=\Gamma( (\mathcal{U},\mathcal{W})\mid x,a)\]
			for $\mathcal{U},\mathcal{W}$ such that $(\mathcal{U},\mathcal{W})=\{x\in\mathbb{R}^d, (x^{i})_{i\in I_U}\in\mathcal{U},(x^{i})_{i\in I_W}\in\mathcal{W} \}\subset\mathcal{X}$ and $(x,a)\in\mathcal{Y}$. 
		\item The action can be decided by $U$, that is for all $\pi\in\Pi$, there exists $\pi_U$ such that
			\[\pi(\{a\}\mid (x^i)_{i=1}^d)=\pi_U(\{a\}\mid (x^i)_{i\in I_U})\]
			for all $(x^i)_{i=1}^d\in\mathcal{X}$.
		\item Finally, we have
			\[\Gamma_U(\mathcal{U}\mid x,a)=\Gamma_U(\mathcal{U}\mid x',a),\quad\forall (x,a),(x',a)\in\mathcal{Y},\mathcal{O}_j(x)=\mathcal{O}_j(x'),\]	
			and for all $(\mathcal{U},\mathcal{W})\subset\mathcal{X}$.
	\end{itemize}
\end{definition}

For a bounded $g:\mathcal{Y}\to\mathbb{R}$, we can embed $g$ into a function $(g)_k$, which maps $\mathcal{Y}^{\mathbb{N}\setminus\{0,1,\dots,j\}}$, defined by
\[(g)_k( (x_{l},a_{l})_{l\ge j+1})=g(x_k,a_k).\]

\begin{definition}[Relative Ignorability]
	We consider the observed operators $(\mathcal{O}_j)_{j\ge0}$ with the observed state space $\mathcal{X}_{j}^{o}$. Let $g:\mathcal{Y}\to\mathbb{R}$ be a bounded function. We say that the missing model is relatively ignorable with respect to $g$ at time $j$ for policy $\pi$ if 
	\begin{align*}
		\mathbb{E}_{j+1,\infty}^\pi\Big[(g)_{j+1}\mid \{x\},\{a\}\Big]=\mathbb{E}_{j+1,\infty}^{\pi}\Big[(g)_{j+1}\mid \{x'\},\{a\}\Big],
\end{align*}
for all $(x,a),(x',a)\in\mathcal{Y}$ and $\mathcal{O}_j(x)=\mathcal{O}_j(x')$.
\end{definition}

\begin{figure}
\centering
\includegraphics[width=300pt, height=200pt]{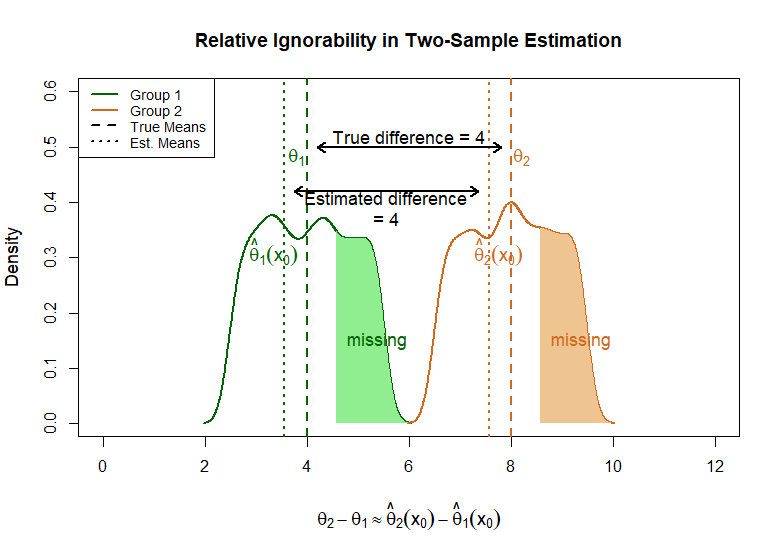}
\label{fig:heuristicri}
\caption[A heuristic example of relative ignorability.]{A heuristic example of relative ignorability in a group mean difference estimation problem. Note that while each estimated group mean is biased, the estimated difference in means is accurate, because missingness affects both groups equivalently.}
\end{figure}

\section{Main Results}

\subsection{Bellman Operator}\label{Bell}

We assume the following:
\begin{enumerate}[label=(A\arabic*),series=assumptions]
	\item \label{AMDP} The Markov Decision Process is partially ingorable for some fixed $I_U$ at any time $j+1$, where $j\ge0$.
	\item \label{Ar} The missing model is relatively ignorable to $\rho$ at any time $j\ge0$ and for all the policy $\pi$ in $\Pi$.
\end{enumerate}
{\bf Example:} Suppose we have $\mathcal{X}\subset\mathbb{R}^4$, $\mathcal{A}\subset\mathbb{R}$. For a $j\ge0$, we have $X_{j+1}=( (X_{j+1}^{i})_{i\in\{1,2\}},(X_{j+1}^{i})_{i\in\{3,4\}})$, and the Markov Decision Process is $(X_{j+1}^{i})_{i\in\{1,2\}}$--partially ignorable. We consider the observed operator at any time $j$ is either \[(x^1,x^2,x^3,x^4)\mapsto (x^1,x^2) \text{ or } (x^1,x^2,x^3,x^4)\mapsto (x^1,x^2,x^3).\]
In other words, in this process, the first two components of the next states is decided by the first two components of the previous states, and the action is also decided based on the first two component, missing $x^3$ or $x^4$ is ignorable.

If the reward is
\[\rho(x^1,x^2,x^3,x^4,a)=\frac{|ax_1|}{|ax_1|+1},\]
then the missing model is relatively ignorable for $\rho$ at any time $j\ge0$ since $\rho$ depends only on the first component of the state.

We back to the general settings, the value function $V_j:\mathcal{Y}^{\mathbb{N}\setminus\{0,1,\dots,j-1\}}\to\mathbb{R}$ at time $j$ is defined by
\begin{align*}
	V_j( (x_k,v_k)_{k\ge j})=\sum_{k=j}^{\infty}\gamma^{k-j}\rho(x_k,a_k).
\end{align*}
Since $\rho$ is bounded and $\gamma\in(0,1)$, the series converges absolutely for given $(x_k,v_k)\in\mathcal{Y}$.

For $j\ge0$, the \emph{action-value} function $q^\pi_j:\mathcal{Y}\to\mathbb{R}$ is defined by
\begin{align*}
	q^\pi_j(x,a)=\rho(x,a)+\gamma\mathbb{E}^\pi_{j+1,\infty}[V_{j+1}\mid \{x\},\{a\}].
\end{align*}
We will have
\begin{align*}
	q^\pi_j(x,a)=\rho(x,a)+\gamma\mathbb{E}^\pi_{j+1,\infty}[(q^\pi_{j+1})_{j+1}\mid \{x\},\{a\}].
\end{align*}

\begin{definition}[Bellman Operator]
\label{def:margbellopt}
For $j\ge0$ and $Q:\mathcal{Y}\to\mathbb{R}$ is a bounded function , the Bellman operator of policy $\pi$ at time $j$ is defined by
\begin{align*}
	T_j^\pi Q(x,a) := \rho(x,a)+\gamma\mathbb{E}^\pi_{j+1,\infty}[(Q)_{j+1}\mid \{x\},\{a\}],
\end{align*}
for $(x,a)\in\mathcal{Y}$. We also define the Bellman optimality operator at time $j$ by
\begin{align*}
	T_j^\ast Q(x,a) := \rho(x,a)+\gamma\sup_{\pi\in\Pi}\mathbb{E}^\pi_{j+1,\infty}[ (Q)_{j+1}\mid \{x\},\{a\}],
\end{align*}
for $(x,a)\in\mathcal{Y}$.

\end{definition}

We have the following lemma:
\begin{lemma}
	$T_j^\pi$ and $T_j^\ast$ are contractive mappings with factor $\gamma$ in the $\sup$-norm.
\end{lemma}

\begin{proof}
	For any two bounded function $Q_1,Q_2:\mathcal{Y}\to\mathbb{R}$, we estimate
	\begin{align*}
		\|T_j^\pi Q_1(x,a)-T_j^\pi Q_2(x,a)\|_{\infty}&=\gamma\sup_{(x,a)\in\mathcal{Y}}\mathbb{E}_{j+1,\infty}^\pi[(Q_1-Q_2)_{j+1}\mid \{x\},\{a\}]\\
		&\le\gamma\sup_{(x,a)\in\mathcal{Y}}\mathbb{E}_{j+1,\infty}^\pi[\|Q_1-Q_2\|_{\infty}\mid \{x\},\{a\}] =\gamma\|Q_1-Q_2\|_{\infty}.
	\end{align*}

	Then, for Bellman optimality operator, we get that
	\begin{align*}
		\|T_j^\ast Q_1(x,a)-T_j^\ast Q_2(x,a)\|_{\infty}&=\|\sup_{\pi\in\Pi}T_j^\pi Q_1(x,a)-\sup_{\pi\in\Pi}T_j^\pi Q_2(x,a)\|_{\infty}\\
		&\le\sup_{\pi\in\Pi} \|T_j^\pi Q_1(x,a)-T_j^\pi Q_2(x,a)\|_{\infty}\le\gamma\|Q_1-Q_2\|_{\infty}.
	\end{align*}
\end{proof}

\begin{lemma}
	Given Assumptions \ref{AMDP} and \ref{Ar}, if the missing model is relatively ignorable with respect to $Q$ at some time $j\ge1$ and for some $\widetilde{\pi}\in\Pi$, and if $I_U\supset I_{j}^o$, then the missing model is relatively ignorable with respect to $T_j^{\widetilde{\pi}} Q$ at any time $k\ge 0$ for all policy $\pi\in\Pi$.

	As a consequence, if the missing model is relatively ignorable with respect to $Q$ at some time $j\ge1$ for all $\pi\in\Pi$, and if $I_U\supset I_j^o$, then the missing model is relatively ignorable with respect to $T_j^\ast Q$ at any time $k\ge0$ for all policy $\pi\in\Pi$.

\end{lemma}
\begin{proof}
	Since the missing model is relatively ignorable with respect to $\rho$ at any time for all policy, we need to show that the missing model is also relative ignorable with respect to
	\[E^{\widetilde{\pi}}_{j}Q(x,a)=\mathbb{E}^{\widetilde{\pi}}_{j+1,\infty}[(Q)_{j+1}\mid \{x\},\{a\}]\]
	at any time and for all policy.

	Since the missing model is relatively ignorable with respect to $Q$ at time $j$ for policy $\widetilde{\pi}$, we have
	\begin{align*}
		\mathbb{E}^{\widetilde{\pi}}_{j+1,\infty}[(Q)_{j+1}\mid \{x\},\{a\}]&=\mathbb{E}^{\widetilde{\pi}}_{j+1,\infty}[(Q)_{j+1}\mid \{x'\},\{a\}],
	\end{align*}
	for all $(x,a),(x',a)\in\mathcal{Y}$ and $\mathcal{O}_j(x)=\mathcal{O}_j(x')$. 

	Thus, we get
	\begin{align*}
		E_j^{\widetilde{\pi}} Q(x,a)=E_j^{\widetilde{\pi}} Q(x',a),
	\end{align*}
	for all $(x,a),(x',a)\in\mathcal{Y}$ and $\mathcal{O}_j(x)=\mathcal{O}_j(x')$. Because $I_U\supset I_j^o$, when $(x^i)_{i\in I_W}$ varies, $T_j^{\widetilde{\pi}}(x,a)$ stays fixed.

For $(x,a),(x',a)\in\mathcal{Y}$ and $\mathcal{O}_{k}(x)=\mathcal{O}_{k}(x')$, we compute
\begin{align*}
	&\mathbb{E}^\pi_{k+1,\infty}[(E_j^{\widetilde{\pi}} Q)_{k+1} \mid \{x\},\{a\}]\\
	&=\frac{\bar{\mu}^\pi_{j}((\mathcal{X}\times\mathcal{A})^{k}\times\{x\}\times\{a\})}{\bar{\mu}_{\infty}^\pi((\mathcal{X}\times\mathcal{A})^{k}\times\{x\}\times\{a\}\times(\mathcal{X}\times\mathcal{A})^{\mathbb{N}\setminus\{0,1,\dots,k\}})}\int_{\mathcal{X}\times\mathcal{A}}E_j^{\widetilde{\pi}}(\hat{x},\hat{a})\pi(d\hat{a}\mid\hat{x})\Gamma(d\hat{x}\mid x,a)\\
	&= \int_{\mathcal{X}\times\mathcal{A}}E_j^{\widetilde{\pi}}(\hat{x},\hat{a})\pi_U(d\hat{a}\mid (\hat{x}^i)_{i\in I_U})\Gamma_U(d(\hat{x}^i)_{i\in I_U}\mid x,a)\Gamma_W(d(\hat{x}^i)_{i\in I_W}\mid x,a)\\
	&= \int_{\mathcal{X}\times\mathcal{A}}E_j^{\widetilde{\pi}}(\hat{x},\hat{a})\pi_U(d\hat{a}\mid (\hat{x}^i)_{i\in I_U})\Gamma_U(d(\hat{x}^i)_{i\in I_U}\mid x',a)\Gamma_W(d(\hat{x}^i)_{i\in I_W}\mid x',a)\\
	&=\mathbb{E}^\pi_{k+1,\infty}[(E_j^{\widetilde{\pi}} Q )_{k+1}\mid \{x'\},\{a\}].
\end{align*}

If the missing model is relatively ignorable with respect to $Q$ for all $\pi\in\Pi$ then
\[\sup_{\pi\in\Pi}\mathbb{E}^{\pi}_{j+1,\infty}[(Q)_{j+1}\mid \{x\},\{a\}]= \sup_{\pi\in\Pi}\mathbb{E}^{\pi}_{j+1,\infty}[(Q)_{j+1}\mid \{x'\},\{a\}],\]
for all $(x,a),(x',a)\in\mathcal{Y}$ and $\mathcal{O}_j(x)=\mathcal{O}_j(x')$.
We perform the previous computation again to obtain that the missing model is relatively ignorable with respect to $T^\ast_j Q$ at any time $k\ge 0$ for all policy $\pi\in\Pi$.
\end{proof}

The lemma means that the Bellman operator is relatively ignorable if we have the additional assumption $I_U\supset I_j^o$. Formally, we assume
\begin{enumerate}[label=(A\arabic*),resume=assumptions]
	\item \label{AIo} For $U$ is the component in Assumption \ref{AMDP}, there exists $j\ge1$ such that $I_U\supset I_j^o$.
\end{enumerate}

Since $T_j^\ast$ is a contractive mapping for all $j\ge0$. We start with $Q=0$, which is relatively ignorable. By the Banach Fixed-Point Theorem, the sequence
\[Q,T_j^\ast Q,(T_j^\ast)^2Q,(T_j^\ast)^3Q,\dots\]
converges in $L^\infty$-norm to the fixed-point $Q_j^\ast$ of $T_j^\ast$ and this sequence is a sequence of relatively ignorable functions.

\subsection{Convergence Under Relative Ignorability}

In Q-learning, we consider the learing factor $\alpha_n\in(0,1)$, which satisfies the following assumption:
\begin{enumerate}[label=(A\arabic*),resume=assumptions]
	\item \label{Aalpha}
		\[\sum_{n=1}^{\infty}\alpha_n=\infty\quad\text{and}\quad \sum_{n=1}^{\infty}\alpha_n^2<\infty.\]
\end{enumerate}

For a given $j\ge0$, the algorithm for $Q$-learning rule is given by
\begin{align*}
	Q^{n}(x,a)=(1-\alpha_n)Q^{n-1}(x,a)+\alpha_n T_j^\ast Q^{n-1}(x,a),
\end{align*}
for the initial function $Q^0=0$.

\begin{theorem}[Q-Learning Convergence]
	For the Markov Decision Process with Assumption \ref{Ar} and the $Q$-learning with Assumption \ref{Aalpha}, the sequence $Q^n(X_j,A_j)$ converges to $Q_j^\ast(X_j,A_j)$ with probability $1$ for any policy $\pi\in\Pi$.
\end{theorem}

\begin{proof}[Proof]
	Fix a policy $\pi\in\Pi$, we consider
\[\delta^n(x,a)= T_j^\ast Q^{n-1}(x,a)-Q_j^\ast(x,a),\]
and
\[\Delta^n(x,a)=Q^{n-1}(x,a)-Q_j^\ast(x,a).\]
From the $Q$-learning rule, we get
\[\Delta^{n+1}(x,a)=(1-\alpha_n)\Delta^n(x,a)+\alpha_n \delta^n(x,a).\]
We need to show that $\Delta^n(X_j,A_j)$ converges to $0$ with probability $1$. 

Since $Q^\ast_j$ is a fixed point of $T^\ast_j$, and $T_j^\ast$ is a contractive map with constant $\gamma$, we estimate that
\begin{align*}
	\|\delta^{n}\|_{\infty}&=\|T_j^\ast Q^{n-1}-T_j^\ast Q^{\ast}_j\|_{\infty}\\
	&\le \gamma\|Q^{n-1}-Q^\ast_j\|_{\infty}=\gamma\|\Delta^n\|_{\infty}.
\end{align*}
As $\mathbb{P}(X_j=x,A_j=a)>0$ for all $(x,a)\in\mathcal{Y}$, we obtain that 
\[\|\delta^{n}(X_j,A_j)\|_{L^\infty(\Omega)}\le\gamma \|\Delta^n(X_j,A_j)\|_{L^\infty(\Omega)}.\]
As $\alpha_n,\gamma\in(0,1)$, we observe that
\[\|\Delta^{n+1}(X_j,A_j)\|_{L^\infty(\Omega)}\le (1-(1-\gamma)\alpha_n)\|\Delta^n(X_j,A_j)\|_{L^\infty(\Omega)}\le e^{-(1-\gamma)\alpha_n}\|\Delta^n(X_j,A_j)\|_{L^\infty(\Omega)}.\]
Hence, inductively, we obtain
\[\|\Delta^{m+1}(X_j,A_j)\|_{L^\infty(\Omega)}\le e^{-(1-\gamma)\sum_{n=1}^{m}\alpha_n}\|\Delta^1(X_j,A_j)\|_{L^\infty(\Omega)}.\]
As $\sum_{n=1}^{\infty}\alpha_n=\infty$, we get
\[\|\Delta^{m+1}(X_j,A_j)\|_{L^\infty(\Omega)}\to0\]
as $m\to\infty$. This leads to $Q^n(X_j,A_j)$ converges with probability $1$ to $Q^\ast_j(X_j,A_j)$.

\end{proof}

Figure \ref{fig:pomdp} visualizes a POMDP. States $x$ transition according to dynamics $\Gamma$, but the agent only observes partial information $x_o$ (with the remaining components $x_m$ missing). The agent maintains belief states $b(x)$ representing probability distributions over possible underlying states. Under relative ignorability conditions, the belief-space policy $\pi$ can still converge to the optimal policy $\pi^*$ through successive applications of the Q-learning update equation\citep{watkins1989} which we denote $B_0$.

\begin{figure}[h]
\centering
\begin{tikzpicture}[
state/.style={circle, draw, minimum size=1.2cm},
    obs/.style={rectangle, draw, minimum size=1.2cm},
    belief/.style={circle, draw, minimum size=1.5cm, minimum width=2.5cm},
    >=stealth
]

% True state space (partially hidden)
\node[state] (s1) at (0,4) {$x_1$};
\node[state] (s2) at (3,4) {$x_2$};
\node[state] (s3) at (6,4) {$x_3$};
\node[state,dashed] (sj) at (9,4) {$x_m$};

% Observations
\node[obs] (o1) at (1.5,2) {$x_{o1}$};
\node[obs] (o2) at (4.5,2) {$x_{o2}$};
\node[obs,dashed] (om) at (7.5,2) {$x_m$};

% Belief states
\node[belief] (b1) at (3,0) {$b(x_1,x_2)$};
\node[belief] (b2) at (7,0) {$b(x_2,x_3,x_j)$};

% State transitions
\draw[->] (s1) -- (s2) node[midway, above] {$\Gamma$};
\draw[->] (s2) -- (s3) node[midway, above] {$\Gamma$};
\draw[->] (s3) -- (sj) node[midway, above] {$\Gamma$};

% Observation model
\draw[->] (s1) -- (o1);
\draw[->] (s2) -- (o1);
\draw[->] (s2) -- (o2);
\draw[->] (s3) -- (o2);
\draw[->] (s3) -- (om);
\draw[->] (sj) -- (om);

% Belief updates
\draw[->, thick] (o1) -- (b1);
\draw[->, thick] (o2) -- (b1);
\draw[->, thick] (o2) -- (b2);
\draw[->, thick] (om) -- (b2);

% Agent's policy
\draw[->, dashed, bend right=30] (b1) to node[midway, left] {$\pi(a|b)$} (0,2);
\draw[->, dashed, bend left=30] (b2) to node[midway, right] {$\pi(a|b)$} (9,2);

% Contraction mapping illustration
\draw[<->, dotted, thick] (b1) to[bend right=30] node[midway, below] {$B_0$} (b2);
\draw[<->, dotted, thick] (b2) to[bend right=45] node[near end, below] {$B_0(B_0)$} (10,0) node[right] {$B^*$};

\end{tikzpicture}
\caption[Visualization of a POMDP.]{Visualization of a POMDP. }
\label{fig:pomdp}
\end{figure}

\subsection{Examples of Relative Ignorability}
\paragraph{Example 1: Clinical Dosing Strategy.}
Consider a clinical decision problem where patient state $X_j = (X_{j,o}, X_{j,m})$ includes observed symptoms and unobserved genetic markers, with action space $\mathcal{A} = \{$chemotherapy, immunotherapy$\}$. Let $Q:\mathcal{X}\times\mathcal{A}\to\mathbb{R}$ be the progression indicating function.

 If genetic markers affect disease progression but not treatment response, then for all possible genetic marker values $x_m^{(1)}, x_m^{(2)} \in \Omega_m$:
 $$\arg\max_{a \in A} Q( (x_{j,o},x_m^{(1)}), a) = \arg\max_{a \in A} Q( (x_{j,o} , x_m^{(2)}), a)=a^\ast(x_{j,o}).$$
In this case, the optimal treatment depends only on the observed symptoms. Thus, one example of a policy \(\pi\) is
\[
\pi^o(a^\ast(x_{j,o}) \mid x_{j,o}) = \pi(a^\ast(x_{j,o}) \mid (x_{j,o}, x_m^{(1)})) = 1
\]
for any \( x_m^{(1)} \in \Omega_m \).
If the next observed symptoms depend only on the current symptoms and the treatment, then the model is partially ignorable.

Furthermore, if we define \( g(x_{j,o}, x_m^{(1)}) = \max_{a \in A} Q\left((x_{j,o}, x_m^{(1)}), a\right) \), then \( g \) is relatively ignorable.

 If genetic markers strongly influence treatment effectiveness, then there exist genetic marker values $x_m^{(1)}, x_m^{(2)}$ such that:
 $$\arg\max_{a \in A} Q( (x_{j,o} , x_m^{(1)}), a) \neq \arg\max_{a \in A} Q( (x_{j,o} , x_m^{(2)}), a)$$
For example, if $x_m^{(1)}$ indicates a mutation making immunotherapy optimal while $x_m^{(2)}$ indicates wild-type genes making chemotherapy optimal, then the same observed symptoms require different treatments based on genetics. In this counterexample, missingness is relatively non-ignorable, and standard Q-learning may converge to a suboptimal point. This scenario might occur if the presence of a wild-type gene is also associated with a phenotype that influences the prescriber's decision.

\paragraph{Example 2: Split-brain Tolerance in Distributed Systems}
Consider a distributed system with nodes containing customer data, where a network partition separates nodes into two groups that cannot communicate. Each partition may continue processing transactions, potentially creating conflicting states. Let $X_j = (X_{j,1}, X_{j,2})$ represent the system state, where $X_{j,1}$ contains data from partition 1 and $X_{j,2}$ contains data from partition 2. During a split brain event, each partition observes only its local state: partition 1 observes $X_{j,o} = X_{j,1}$ while $X_{j,m} = X_{j,2}$ is missing, and vice versa.

Not all inconsistencies matter equally for different application functions: Only data that is both conflicting and used by the specific application function being evaluated is relatively non-ignorable with respect to that function's decision-making.

Systems architects can use relative ignorability to implement \textit{selective degradation}: during split brain scenarios, functions that depend on relatively non-ignorable conflicting data can be temporarily disabled or routed to require manual approval, while functions that operate on relatively ignorable conflicts can continue operating normally.

Let $\hat{g}(X_j, A_j)$ represent an application function (such as ``approve customer transaction'' or ``calculate account balance'') that takes system state and proposed action as inputs. The missing components $X_{j,m}$ are relatively ignorable with respect to $\hat{g}$ if Equation \ref{eq:dsrl} holds (where $X_{j,o}$ is the observed information and $\Omega_m$ is the set of possible values $X_{j,m}$ can take), even when $X_{j,m}$ contains conflicting information.

\begin{equation}
\hat{g}( (X_{j,o},X_{j,m}), A_j) = \hat{g}( (X_{j,o},X'_{j,m}), A_j) \forall X_{j,m},X'_{j,m} \in \Omega_{m}
\label{eq:dsrl}
\end{equation}

 Assuming one has knowledge of the data dependencies for each application function, one could track which data elements are conflicting across partitions, and disable only the functions which use the conflicting information. This selective termination process would allow unaffected functions to continue to run, thereby minimizing the impact to the application while also mitigating the impact of information errors.

Consider, for example, a payment processing system which includes the following functions: i) Fraud detection, which uses the customer's transation history, and ii) Marketing recommendations, which is based on transaction history as well as the user profile (age, gender, address, etc.) iii) other capabilities such as purchasing, which uses neither transaction history or demographic information. Let us use $x'_j$ to denote the transaction history at time $j$ and $x''_j$ to denote the user's demographic data. 

Suppose that $x''_j$ is stored on both partitions and differs due to the split-brain problem. Traditional methods for handling this situation might include temporarily disabling the user's account to avoid compounding errors, since $x''_j$ is classically non-ignorable with respect to the application function for that user. However, this might not be necessary: Since the fraud detection capability and other functionality does not depend on $x''_j$, we might simply disable the marketing recommendations and allow the rest of the application functions to run normally. We can do this because $x''_j$ is \emph{relatively ignorable} with respect to fraud detection (and other functions).

\section{Discussion}

Classical Q-learning convergence theory \citep{watkins1989, jaakkola1994, tsitsiklis1997}, and its extensions to function approximation \citep{mnih2015} assume complete state observability. We have developed, here, a framework of relative ignorability that relaxes this fundamental requirement while preserving convergence guarantees. Our result eliminates the need for explicit POMDP modelling and estimation in certain cases, by specifying when such complexity is unnecessary.

Recent work on causal reinforcement learning \citep{zhang2020} shares similar motivations but focuses on confounding variables rather than general missing data patterns; our relative ignorability framework provides a more general unifying perspective on when partial observability can be safely ignored. Advantage learning \citep{harmon1996residual} also naturally connects to relative ignorability. Advantage learning focuses on learning the advantage function $A(x, a) = Q(x, a) - V(x)$, which represents relative action values rather than absolute Q-values. From a relative ignorability perspective, if missing components affect all actions equally, they may bias individual Q-values while preserving advantage rankings. The results described here could also be extended to show that advantage learning is more robust to certain types of missing data than standard Q-learning. Our results show that the Markov assumption can be relaxed when violations do not affect decision-making. This suggests that the observability requirements arise dynamically relative to task demands.

 There are a variety of future directions for our work. First, our theoretical result might be extended to Deep Q-learning, a popular modification of classical Q-learning which leverages Deep Neural Networks in order to estimate the Q-function. To probe the feasibility of this direction, we conducted a simulation which compared the performance of Deep Q-Learning under various relative ignorability conditions. The simulated environment consisted of a classic 2x2 gridworld with one goal square as well as a trap square, and potential actions consisted of moving left, right, up or down. In addition, we added a latent ``mode" variable which affected the layout: The complete state is $X_j = (position, mode)$ where $position \in \{(0,0), (0,1), (1,0), (1,1)\}$ and $mode \in \{0,1\}$. The agent only observes position; mode is missing. We consider two forms of the latent mode, one which is relatively ignorable, and one which is relatively non-ignorable. The relatively ignorable mode affects reward values only: if mode=1, the goal square yields 10 reward (8 if mode=0), and trap yields -10 penalty (-8 if mode =0). Note that since mode affects the reward even after conditioning on the observed state, the unobserved mode is classically non-ignorable, though relatively ignorable. Conversely, the relatively non-ignorable mode swaps the location of the goal with the location of the trap, consistently returning 10 reward at the goal and -10 at the trap. Additionally, the agent was penalized with -0.1 reward for each step, to incite efficient progress toward the goal.

 We performed Deep Q-learning using a two-layer neural network, with ReLU activation between the two layers. Each layer consisted of 64 nodes. For comparison with POMDP learning, we also consider a situation where the agent has access to a noisy signal of the mode. We generated this noisy signal as a random normal variable, updated at each timestep, with the mean equal to the value of the latent mode (0 or 1), and variance $\sigma^2 = 0.15$. The code to implement the environment step update follows:
 \begin{verbatim}
 def step(self, action_idx):
        action = ACTIONS[action_idx]
        dx, dy = ACTION_TO_DELTA[action]
        new_x = np.clip(self.pos[0] + dx, 0, GRID_SIZE - 1)
        new_y = np.clip(self.pos[1] + dy, 0, GRID_SIZE - 1)
        self.pos = (new_x, new_y)
        
        reward = -0.1
        done = False
        if self.pos == (0, 1):  # Goal
            reward = (10 if self.mode == 0 else 9) if self.relative_ignorability 
                else (-10 if self.mode == 0 else 10)
            done = True
        elif self.pos == (1, 0):  # Trap
            reward = -10 if self.relative_ignorability 
                else (10 if self.mode == 0 else -10)
            done = True

        # observe new evidence
        self.observed_variable = np.random.normal(loc=self.mode, scale=0.15)
    
        # update belief based on new observation
        self.update_belief(obs=self.observed_variable)
    
        return np.array([*self.pos, self.belief_mode],
        
        dtype=np.float32), reward, done
 \end{verbatim}

 For POMDP estimation, we performed a classic Bayesian belief update at each step as follows. Here, the belief mode is the agent's prior belief of what the current mode is, where prior belief is intialized agnostically to 0.5. 

\begin{verbatim}
def update_belief(self, obs):
    p1 = np.exp(-(obs - 1)**2 / 2)
    p0 = np.exp(-(obs - 0)**2 / 2)
    prior = self.belief_mode
    self.belief_mode = (p1 * prior) / (p1 * prior + p0 * (1 - prior) + 1e-8)
\end{verbatim}

We trained all models for 1000 episodes, using $\gamma = 0.9$, $\alpha = .001$, and epsilon decay rate $0.995$. To ensure stability of the results, we repeated the simulation with 5 different random seeds, and averaged the resulting reward curve for each mode (0 or 1, with 1 corresponding to the relatively non-ignorable case), model (POMDP or Vanilla), and training episode (1 to 1000) across the seeds. 

Figure \ref{fig:rl_gridworld} shows the 50-timepoint rolling mean of the 5-seed average reward curve from each algorithm and mode across training epochs. Despite the nonignorably missing ``mode" information, rewards obtained by vanilla Q-learning converge to 9: Note that 9 is the maximum possible reward, since it is the average of 8 and 10 (the goal rewards under each mode in the relatively ignorable setting). Moreover, vanilla Q-learning appears to be more efficient than POMDP in the relatively ignorable case, converging in fewer epochs. The slower convergence of POMDP Q-learning is understandable from an information theoretic perspective: In POMDP learning, the information in each observation must be shared across two estimation procedures - belief update and Q update - while in Vanilla Q-learning there is only one function to estimate. However, if the mode is non-ignorable, then Vanilla Q-learning does not converge, while POMDP estimation can still find the optimal policy.

 \begin{figure}[ht]
\centering
\includegraphics[width=350pt, height=180pt]{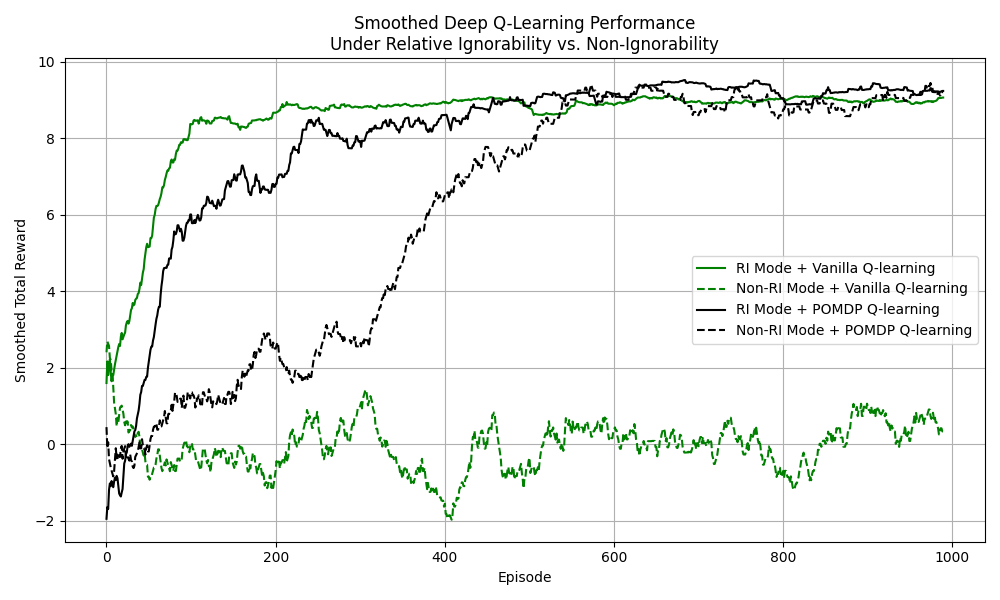}
\label{fig:rl_gridworld}
\caption{Reward curves from Vanilla and POMDP Q-learning under relatively ignorable and relatively non-ignorable latent mode. Vanilla Q-learning converges to the maximum obtainable reward faster than POMDP Q-learning under relative ignorability.}
\end{figure}

 In addition to Deep Q-learning, our framework could also extend to continuous state spaces, function approximation, and policy gradient methods. Connections to optimal control theory \citep{bertsekas2019} suggest further theoretical developments. In practice, determining relative ignorability requires domain knowledge or empirical validation. Future work should develop algorithms for automatically detecting when this condition holds. Methods based on the index of sensitivity to nonignorability developed by Troxel, Ma, and Heitjan \cite{troxel2004index} seem compelling.

\section{Conclusion}

We have introduced relative ignorability as a condition under which Q-learning converges despite missing state components. This framework relaxes classical assumptions while maintaining theoretical guarantees, offering a middle ground between full observability and complex POMDP solutions. Our results suggest that the curse of dimensionality in agentic AI may be mitigated by focusing on decision-relevant information rather than complete state reconstruction.

\bibliographystyle{plain}
		
\bibliography{bibliography}

\iffalse

\fi
%\pagebreak
%\input{app.tex}
\end{document}